\documentclass[letterpaper]{article} 
\usepackage{aaai24}  
\usepackage{times}  
\usepackage{helvet}  
\usepackage{courier}  
\usepackage[hyphens]{url}  
\usepackage{graphicx} 
\usepackage{natbib}  
\usepackage{caption} 
\usepackage{algorithm} 
\usepackage{algorithmic}  
\usepackage[algo2e,noline,ruled,]{algorithm2e} 
\usepackage[utf8]{inputenc}
\usepackage{amsmath}
\usepackage{amsthm}
\usepackage{amsfonts,amssymb}
\usepackage{xcolor}
\usepackage{array}
\usepackage{lipsum}
\usepackage{numprint}
\usepackage{multirow}

\usepackage[capitalise]{cleveref}
\usepackage{subcaption}
\usepackage{forest}
\usepackage{enumitem} 
\usepackage[export]{adjustbox}

\usepackage{newfloat}
\usepackage{listings}
\DeclareCaptionStyle{ruled}{labelfont=normalfont,labelsep=colon,strut=off} 
\floatstyle{ruled}
\newfloat{listing}{tb}{lst}{}
\floatname{listing}{Listing}
\title{Minimal Macro-Based Rewritings of Formal Languages: \\Theory and Applications in Ontology Engineering (and beyond)}

\author {
    Christian Kindermann\textsuperscript{\rm 1},
    Anne-Marie George\textsuperscript{\rm 2},
    Bijan Parsia\textsuperscript{\rm 3},
    Uli Sattler\textsuperscript{\rm 3}
}
\affiliations {
    \textsuperscript{\rm 1} Biomedical Informatics, Stanford University, United States of America \\
    \textsuperscript{\rm 2} Department of Informatics, University of Oslo, Norway\\
    \textsuperscript{\rm 3} Department of Computer Science, University of Manchester, United Kingdom\\
    christian.kindermann@stanford.edu, annemage@ifi.uio.no, 
    $\{$bijan.parsia, uli.sattler$\}$@manchester.ac.uk
}

\newtheorem{theorem}{Theorem}[section]
\newtheorem{definition}{Definition}[section]
\newtheorem{lemma}{Lemma}[section]

\newtheorem{example}{Example}
\newcounter{probCounter}
\newtheorem{problem}[probCounter]{Problem}

\newcommand{\Lang}{\mathcal{L}}
\newcommand{\conc}[1]{\ensuremath{\mathsf{#1}}}

\newcommand{\opr}[1]{\ensuremath{\operatorname{\mathit{#1}}}}

\newcommand{\isa}{\sqsubseteq}

\usepackage{etoolbox}
\usepackage{marginnote}
\usepackage[textsize=scriptsize]{todonotes} 
\newcommand{\mytodo}[2]{\todo[size=\tiny, color=#1!50!white]{#2}}
\newcommand{\myrevtodo}[2]{{%
		\let\marginpar\marginnote
		\reversemarginpar
		\renewcommand{\baselinestretch}{0.8}%
		\todo[size=\tiny, color=#1!50!white]{#2}\xspace}}
\newcommand{\myinlinetodo}[2]{\todo[size=\small, color=#1!50!white, inline, caption={}]{#2}\xspace}
\newcommand{\registerAuthor}[3]{%
	\expandafter\newcommand\csname #2com\endcsname[1]{\mytodo{#3}{\textsc{#2}: 
	##1}}%
	\expandafter\newcommand\csname 
	#2revcom\endcsname[1]{\myrevtodo{#3}{\textsc{#2}: ##1}}%
	\expandafter\newcommand\csname 
	#2inline\endcsname[1]{\myinlinetodo{#3}{\textsc{#2}: ##1}}%
	\expandafter\newcommand\csname 
	#2inlineLater\endcsname[1]{\lv{\myinlinetodo{#3}{\textsc{#2}: ##1}}}%
}

\registerAuthor{Chris}{ck}{green!50!gray}

\registerAuthor{Marie}{mg}{green!40!blue!25!white}

\newcommand{\size}{\mathit{size}}
\newcommand{\occ}[1]{\mathit{occ}_{#1}}
\newcommand{\argmin}{\mathit{arg min}}

\newcommand{\macrosystem}{\mathfrak{M}}

\newcommand{\terms}{T(\Sigma)}
\newcommand{\LL}{\mathcal{L}}
\newcommand{\mterms}[1]{T(\Sigma \cup #1)}
\newcommand{\encode}{\mathfrak{M}} 
\newcommand{\fencode}[2]{(#1_{#2}, \mathcal{#2})} 
\newcommand{\MM}{\mathcal{M}} 
\newcommand{\M}{\ensuremath{\mathcal{M}}\xspace}
\newcommand{\macrodefinitions}{\mathcal{M}} 

\newcommand{\MMstar}{\MM^*} 
\newcommand{\fixedpoint}[1]{\MM^*(#1)} 
\newcommand{\minencode}[2]{\encode_{\textit{min}}^{#1,#2}} 
\newcommand{\ont}{\ensuremath{\mathcal{O}}\xspace}

\newcommand{\appsymb}{$\bigstar$}
\newcommand{\appref}[1]{{\appsymb}}

\newcommand{\toappendixx}[1]{}

\begin{document}

\maketitle

\begin{abstract}
In this paper, we introduce the problem of rewriting finite formal languages using syntactic macros such that the rewriting is minimal in size. We present polynomial-time algorithms to solve variants of this problem and show their correctness. To demonstrate the practical relevance of the proposed problems and the feasibility and effectiveness of our algorithms in practice, we apply these to biomedical ontologies authored in OWL. We find that such rewritings can significantly reduce the size of ontologies by capturing repeated expressions with macros.
In addition to offering valuable assistance in enhancing ontology quality and comprehension, the presented approach introduces a systematic way of analysing and evaluating features of rewriting systems (including syntactic macros, templates, or other forms of rewriting rules) in terms of their influence on computational problems.
\end{abstract}

\section{Introduction}

KISS and DRY, standing for ``Keep It Simple, Stupid'' and ``Don't Repeat Yourself'',
are key principles in software development.
These principles prove advantageous in other contexts involving formal languages, including knowledge bases, formal specifications, and logical theories.

Despite the value of simplification and reduction of repetition, there is no universally applicable mechanism assisting with the associated refactoring of formal languages. Challenges arise from the need to preserve application-specific properties and the difficulty of specifying desirable and undesirable refactorings.
In other words, there is a fine line between succinct simplicity and obscure terseness.

This paper proposes a formal framework for refactoring formal languages through automated conversions into smaller rewritings. The focus is on identifying rewriting mechanisms applicable in diverse situations and analysing their properties.
  One simple but effective refactoring strategy is the introduction of \textit{names} for more complex recurring structures. Preserving syntactic structure is often important, motivating our exploration of rewriting mechanisms based on \textit{syntactic macros}.

Consider the following example language consisting of formulas in propositional logic $\Lang = \{a \land b, (a \land b) \lor c,~((a \land b) \lor c) \land d \}$. A macro $m \mapsto a \land b$ can be used to ``macrofy'' $\Lang$ into 
$\Lang_m = \{m, m \lor c,~(m \lor c) \land d \}$. Note that the expression $m \lor c$ occurs twice in the macrofied language $\Lang_m$. So, the question arises whether a rewriting mechanism allows for \textit{nested} macro definitions, e.g., $m' \mapsto m \lor c$, or not.

Deciding what features to include in a rewriting mechanisms, e.g., allowing for nesting or not, is not straightforward.
To evaluate a feature's impact, we propose to analyse its influence on computational problems \textit{formulated w.r.t.\ rewriting mechanisms}. For instance, consider the above example and the problem of computing size-minimal rewritings.
If the complexity of computing minimal rewriting with un-nested macros is lower than that of minimal rewriting with nested macros, then this effect of the feature ``nesting'' can inform its inclusion for practical purposes.

In this paper, we investigate a rewriting mechanism based on syntactic macros allowing for nesting.
We show conditions under which the problem of finding size-minimal rewritings can be solved in polynomial time. We implement this rewriting mechanism and apply it to find size-minimal rewritings of large biomedical ontologies of practical relevance. The attained size reductions are comparable to, if not better than those of existing approaches, while also showcasing superior performance in terms of processing time.

\section{Preliminaries}\label{sec:preliminaries}

We define the technical terms used in the following sections. In particular, we introduce terms over a mixture of ranked, unranked, and mixed symbols, which will then enable us to apply our framework to the Web Ontology Language (OWL) which uses all three kinds of symbols.

\newcommand{\ar}{\conc{Ar}}

\begin{definition}[Alphabet, Terms, Language]\label{def:alphabetTermsLanguage}
        A \emph{ranked alphabet} $\mathcal{F}$ is a finite set of \emph{symbols} together with a function 
         $\ar: \mathcal{F}\rightarrow \mathbb{N}$ that associates symbols with their \emph{arity}. Symbols $f\in \mathcal{F}$ with $\ar(f)=0$ are called \emph{constants} and symbols with $\ar(f)=1$ are called \emph{unary}.
        An \emph{unranked alphabet} $\mathcal{U}$ is a set of symbols. 
        A \emph{mixed alphabet} $\mathcal{X}\subseteq  \mathcal{F}$ is a subset of the ranked alphabet. 
        An \emph{alphabet} $\Sigma = \mathcal{F} \cup \mathcal{U}$ is a set of ranked or unranked symbols.
        
        The set of \emph{terms} $T(\Sigma)$ over an alphabet $\Sigma = \mathcal{F} \cup \mathcal{U}$ is  inductively defined as follows: 
        \begin{itemize}
        \item $f\in T(\Sigma)$ if $f \in\mathcal{F}$ and  $\ar(f) = 0$, 
        \item  $f(t_1, \ldots, t_n) \in T(\Sigma)$ if $f\in\mathcal{F}$, $n\geq 1$, and  $\ar(f) = n$,    
        \item  $f(t_1, \ldots, t_n) \in T(\Sigma)$ if $f\in\mathcal{U}$ and $n\geq 1$, and  
        \item  $f(t_1, \ldots, t_n) \in T(\Sigma)$ if $f\in\mathcal{X}$, $n\geq 1$, and  $\ar(f) \leq n$,
        \end{itemize}
        where $t_1, \dots, t_n \in T(\Sigma)$.
        A \emph{language $\Lang$ over }$\Sigma$ is a set of terms $\Lang\subseteq T(\Sigma)$.
\end{definition}

To deal with unranked and mixed symbols, in particular to define what it means for two terms to be identical, we introduce the notion of a term tree where edge labels are used to treat terms uniformly as unordered trees. 

\newcommand{\ttree}{\mathsf{Tr}}
\begin{definition}[Term Tree]\label{def:termTree}
            Let $\Sigma = \mathcal{F}\cup \mathcal{U}$ be an alphabet  of disjoint sets of ordered  ($\mathcal{F}$), mixed  ($\mathcal{X}\subseteq \mathcal{F}$), and unordered   ($\mathcal{U}$) symbols. 
            The \emph{term tree} $\ttree(t)$ of a term $t \in T(\Sigma)$ is an unordered tree  inductively defined as follows: if 
            \begin{itemize}
                \item  $t$ is a constant, then $\ttree(t)$ is a node labelled with $t$,
                \item  $t = f(t_1, \ldots, t_n)$ and $f \in \mathcal{F}$ with $\ar(f) = n$, then $\ttree(t) = $
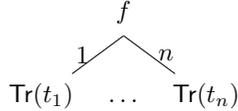
                \begin{adjustbox}{width=0.185\textwidth,valign=t}
                \begin{forest}
                        for tree={
                          l sep=15pt,
                          parent anchor=south,
                          align=center
                        }
                        [$f$
                            [$\ttree(t_1)$,edge label={node[midway,left]{$1$}}] 
                             [$\ldots$, no edge ]
                            [$\ttree(t_n)$,edge label={node[midway,right]{$n$}}]
                        ]
                        \end{forest}            
                \end{adjustbox}
                \item  $t = f(t_1, \ldots, t_n)$ and $f \in \mathcal{U}$,
                then $\ttree(t) =$ \begin{adjustbox}{width=0.185\textwidth,valign=t}
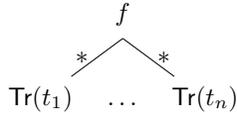
                \begin{forest}
                        for tree={
                          l sep=15pt,
                          parent anchor=south,
                          align=center
                        }
                        [$f$
                            [$\ttree(t_1)$,edge label={node[midway,left]{$*$}}] 
                             [$\ldots$, no edge ]
                            [$\ttree(t_n)$,edge label={node[midway,right]{$*$}}]
                        ]
                        \end{forest}            
                \end{adjustbox}
                  \item  $t = f(t_1, \ldots, t_m, \ldots, t_n)$ and $f \in \mathcal{X}$ with $\ar(f) = m$,
                then $\ttree(t) = $ \begin{adjustbox}{width=0.3\textwidth,valign=t}
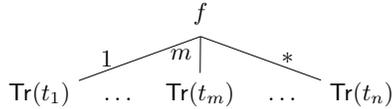
                \begin{forest}
                        for tree={
                          l sep=15pt,
                          parent anchor=south,
                          align=center
                        }
                        [$f$
                            [$\ttree(t_1)$,edge label={node[midway,left]{$1~~~$}}]
                             [$\ldots$, no edge ]
                             [$\ttree(t_m)$,edge label={node[midway,left]{$m$}}] 
                             [$\ldots$, no edge ]
                            [$\ttree(t_n)$,edge label={node[midway,right]{$~~*$}}]
                        ]
                        \end{forest}            
                \end{adjustbox}
            \end{itemize}
        \end{definition}

                The notion of \textit{term equality} then corresponds to the notion of graph isomorphism between term trees that preserves both node and edge labels.
        Similarly, the notion of a \textit{subterm} relation is based on graph isomorphisms between restrictions of their corresponding term trees.
                    
        \begin{definition}[Subterms / Subtrees]\label{def:subterm}
Let $t, t'$ be two edge-and node-labelled, unordered trees. We say that $t$ is a \emph{subtree of }$t'$ if there exists a node $n$ in $t'$ and a (node- and edge-label preserving) isomorphism from the subtree of $t'$ rooted at $n$ to $t$. In this case, we refer to $n$ as a \emph{position} of $t$ in $t'$ and use $t'|_n$ for the subtree of $t'$ rooted in $n$.

Let $t,t' \in T(\Sigma)$ be two terms. We say that $t$ is a \emph{subterm of}  $t'$, denoted $t \preceq t'$, if $T(t)$ is a subtree of $T(t')$. We say that 
$t$ and $t'$ \emph{are equal}, denoted $t = t'$, if they are both subterms of each other. 
\end{definition}
        
        In the following, we rarely need to distinguish between terms and their corresponding trees and use \textit{tree} and \textit{term} interchangeably;  we will not mention whether symbols are ordered, unordered, or mixed unless the context requires this.

        \begin{definition}[Substitution]\label{def:substitution}
        The \emph{substitution} of a term $u \in T(\Sigma)$ for a term $t' \preceq t$ at position $p$, written $t[u]_p$, is the replacement of $t'$ in $t$ at position $p$ with $u$.
\end{definition} 

     \subsection{Web Ontology Language}\label{sec:owl}

We will evaluate the feasibility and effectiveness of the algorithms introduced later on OWL ontologies~\cite{DBLP:journals/ws/GrauHMPPS08, Motik2008OWL2W}. If readers are not familiar with OWL, it suffices to say that OWL has all three kinds of symbols we consider here, i.e., ordered, unordered, and mixed symbols.  
For those familiar with OWL, we use some Description Logic syntax for the sake of readability but interpret logical constructors as specified by OWL.
Figure~\ref{fig:OWLexample} shows two example term trees for the OWL expressions
$\opr{SubClassOf}(\conc{A},\opr{ObjectIntersectionOf}(\conc{B},\conc{C}))$
and
$\opr{DisjointUnion}(\conc{A},\conc{B},\conc{C},\conc{D})$.
Note that \opr{SubClassOf}, \opr{ObjectIntersectionOf}, and \opr{DisjointUnion} are, respectively, 
ordered, unordered, and mixed symbols  and that we have used descriptive names rather than numbers to indicate the order of terms.

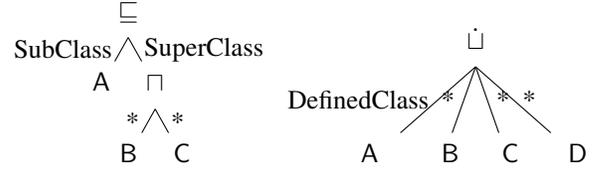
\begin{figure}[t]
  \centering

      \begin{forest}
        for tree={
          l sep=9pt,
          parent anchor=south,
          align=center
        }
        [$\isa$
    [$\conc{A}$, edge label={node[midway,left]{SubClass}}]
        [$\sqcap$, edge label={node[midway,right]{SuperClass}}
        [$\conc{B}$, edge label={node[midway,left]{*}} ] 
        [$\conc{C}$, edge label={node[midway,right]{*}}  ]
        ]
        ] 
      \end{forest}
      \begin{forest}
        for tree={
          l sep=25pt,
          parent anchor=south,
          align=center
        }
        [$\dot\sqcup$
        [~$\conc{A}~~$, edge label={node[midway,left]{DefinedClass}}]
        [~$\conc{~B}$, edge label={node[midway,left]{*}} ] 
        [~$\conc{C}$, edge label={node[midway,right]{*}} ]
        [~$\conc{~D}~~$, edge label={node[midway,right]{*}}]
        ] 
      \end{forest}

  \caption{Example of term trees for OWL expressions.}
  \label{fig:OWLexample}
\end{figure}
\section{Macros and Encodings}

\subsection{Macros for Formal Languages}\label{sec:basicNotions}

A \textit{macro} is understood as a rule that maps an input (in our case finite sets of terms, possibly using macros symbols) to an output  (also terms).
Such a mapping is usually used to expand a smaller input into a larger output.

\begin{definition}[Macro Definition]\label{def:macro}
    Let $M$   be a finite set of constant \emph{macro symbols} disjoint from $\Sigma$.  
        A \emph{set of macro definitions} is a total function $$\M \colon \begin{array}[t]{rcl} M &\rightarrow &T(\Sigma \cup M)\setminus (\Sigma \cup M) \\
        m &\mapsto &t
        \end{array}$$
        If $\M(m) = t$, we call $m \mapsto t$ a \emph{macro definition}
        and $m$  \emph{a macro for} $t$.
\end{definition}

The range of macro definitions is a set of non-constant terms. 
In the following, we omit the explicit exclusion of non-constant terms and simply write $\M \colon M \rightarrow T(\Sigma \cup M)$ to specify macro definitions for improved readability.

A macro's occurrence can be \emph{expanded} to its associated output structure, i.e., by replacing the macro with its expansion. Such an expansion may lead to the occurrence of other macros, giving rise to an iterative expansion process. The limit of this expansion process will be referred to as  \emph{the expansion} of a macro. 

\begin{definition}[Macro Expansion]\label{def:macroExpansion}
    Let $\mathcal{M} \colon M \rightarrow T(\Sigma \cup M)$ be a set of macro definitions and $t \in T(\Sigma \cup M)$ a term. 
      The $1$-\emph{step expansion} $\M(t)$ of $t$ w.r.t.\ $\mathcal{M}$ is the term $t'$ obtained by replacing all occurrences of any macro $m\in M$ in $t$ with its expansion $\M(m)$. 
    The $n$-\emph{step expansion} of $t$ w.r.t.\ $\mathcal{M}$ is $\mathcal{M}^n(t) = \underbrace{\mathcal{M}(\ldots\mathcal{M}}_{n \text{ times}}(t))$.
    The \emph{expansion} of a macro, written $\mathcal{M}^*(t)$, is the least fixed-point of $\mathcal{M}$ applied to $t$.
\end{definition}

The fixed-point $\fixedpoint{m}$ of a macro $m$ does not necessarily exist.
Consider the macro definition $m \mapsto succ(m)$. While the $n$-step expansion of $m$ is well-defined for any $n \in \mathbb{N}$, 
the expansion $\fixedpoint{m}$  does not exist.

The inverse of macro expansion is macro \textit{instantiation}, i.e., replacing a term with a macro symbol that can be expanded to the term.

\begin{definition}[Macro Instantiation]\label{def:macroInstantiation}
Let $\Lang \subseteq T(\Sigma \cup M)$ be a language and $\M \colon M \rightarrow T(\Sigma \cup M)$ a set of macro definitions. A macro $m \in M$ \emph{can be instantiated in }$\Lang$ if there exists an $n$ s.t. $\M^n(m)$ occurs in $\Lang$, i.e., there is a $t\in \Lang$ and a position $p$ with $t|_p = \M^n(m)$. In this case, we also say  
\emph{$m$ can be instantiated in $t$}. We refer to the act of replacing $t|_p$ with $m$ as \emph{instantiating} $m$ in $t$ at position $p$. The term $m$ occurring at position $p$ in $t[m]_p$ is an \emph{instantiation} of $m$.
\end{definition}

Since we are interested in macros for terms,
we will restrict macro definitions so that the expansion of a macro is required to exist, for which we introduce some terminology. 

\begin{definition}[Acyclic Macro Definitions]\label{def:acyclic}
    For macros $m, m' \in M$, we say that $m$ \emph{directly uses} $m'$ if $m'$ occurs in $\M(m)$. Next, let \emph{uses} be the transitive closure of \emph{directly uses}. A set of macro definitions $\M\colon M \mapsto T(\Sigma \cup M)$  is \emph{acyclic} if there are no `uses' cycles in $M$, i.e., if there  is no macro $m\in M$  that uses itself. 
\end{definition}

\begin{lemma}\label{lemma:cycles}
    Let $\M\colon M \mapsto T(\Sigma \cup M)$ be a set of macro definitions. Then: 
    \begin{enumerate}
        \item[] \hspace*{-0.4cm}(a) \M is acyclic iff $\fixedpoint{t}$ exists for each $t\in  T(\Sigma \cup M)$.
        \item[] \hspace*{-0.4cm}(b) If $\fixedpoint{t}$ exists, it is unique.
        \item[] \hspace*{-0.4cm}(c) If $\fixedpoint{t}$ exists then $\fixedpoint{t}\in T(\Sigma)$.
    \end{enumerate}
\end{lemma}

So, acyclicity guarantees the existence of least fixed points for the expansion of macros. Lemma~\ref{lemma:cycles}(a) is an immediate consequence of the close relationship between macro systems and finite, ground term rewriting systems \cite{DBLP:books/daglib/0092409}. Lemma~\ref{lemma:cycles}(b) is due to $\M$ being a function, and  Lemma~\ref{lemma:cycles}(c) is due to $\fixedpoint{t}$ being a fixed-point and \M being total on $M$. 
In what follows, we restrict our attention to acyclic sets of macro definitions $\M\colon M \mapsto T(\Sigma \cup M)$. 

Augmenting a language $\Lang \subseteq T(\Sigma)$ with a set of macro symbols $M$
can be seen to define a new language $\Lang_M \subseteq T(\Sigma \cup M)$.
If the expansion of all macros in $\Lang_M$ w.r.t.\ some macro definitions $\mathcal{M} \colon M \rightarrow T(\Sigma \cup M)$ yields $\Lang$, then $\Lang_M$ together with $\mathcal{M}$ can be seen as an \textit{encoding} of $\Lang$.

\begin{definition}[Macro System, Macrofication, Language Encoding]\label{macroSystem}
    A \emph{macro system} $\macrosystem$ is a tuple $(\Lang_M, \macrodefinitions)$
    where $\Lang_M \subseteq T(\Sigma \cup M)$ is a language
    and $\macrodefinitions \colon M \rightarrow T(\Sigma \cup M)$  a set of macro definitions.
    If $\fixedpoint{\Lang_M} := \bigcup\limits_{t \in \Lang_M} \fixedpoint{t} = \Lang$,
    then $\macrosystem$ is an \emph{encoding} of  $\Lang \subseteq T(\Sigma)$. 
    In this case, we will refer to $\Lang_M$ as a \emph{macrofication} of $\Lang$ w.r.t.\ $\mathcal{M}$.
\end{definition}

\begin{example}\label{ex: 1}
    Consider the language $\LL \subseteq \terms$ over symbols $\Sigma = \{a,b,c,d,e,f\}$ which is given by the set of terms
    \[\LL = \{a(b(c(e),d(f))), 
    b(c(e),d(f)), a(d(f))\}.\]
    Furthermore, consider the set of macro definitions 
    \[\MM = \{m\mapsto c(e), m' \mapsto d(f), m'' \mapsto b(m,d(f))\}.\]
    We can represent the terms in $\LL$ and in the macro definitions of $\MM$ as trees in the following way: 
    
\noindent
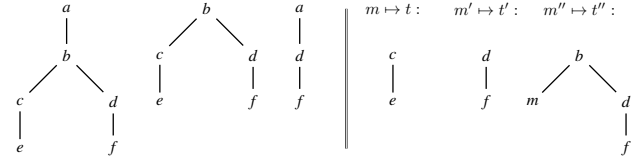
\resizebox{\linewidth}{!}{
\begin{tikzpicture}
\foreach \pos/\name/\labl in {
{(1,3)/a1/a}, {(1,2)/b1/b}, {(0,1)/c1/c}, {(2,1)/d1/d}, {(0,0)/e1/e}, {(2,0)/f1/f},
{(4,3)/b2/b}, {(3,2)/c2/c}, {(3,1)/e2/e}, {(5,2)/dd/d}, {(5,1)/ff/f}, 
{(6,3)/a2/a}, {(6,2)/d2/d}, {(6,1)/f2/f}}
        \node[] (\name) at \pos {\labl};

\foreach \source/ \dest in {
a1/b1, b1/c1, c1/e1, b1/d1, d1/f1,
b2/c2, c2/e2, b2/dd, dd/ff,
a2/d2, d2/f2}
        \path[draw,thick] (\source) -- (\dest);

\path[draw,double] (7,3) -- (7,0);

\foreach \pos/\name/\labl in {{(8,3)/m/$m \mapsto t: $}, {(10,3)/m'/$m' \mapsto t': $}, {(12,3)/m''/$m'' \mapsto t'': $}}
        \node[] (\name) at \pos {\labl};
        
\foreach \pos/\name/\labl in {
{(8,2)/c3/c}, {(8,1)/e4/e},
{(10,2)/d3/d}, {(10,1)/f3/f}, 
{(12,2)/b3/b}, {(11,1)/mm/m}, {(13,1)/d4/d}, {(13,0)/f4/f}}
        \node[] (\name) at \pos {\labl};

\foreach \source/ \dest in {
c3/e4,
d3/f3,
b3/mm, b3/d4, d4/f4}
        \path[draw,thick] (\source) -- (\dest);
\end{tikzpicture}
}   
    The 1-step expansion of $m''$ is the term $b(m,d(f))$ in which $m$ is instantiated. Furthermore, the expansion (which in this case is the 2-step expansion) of $m''$ is $\MMstar(m'') = b(c(e),d(f))$.
    So, the macro system $\encode = \fencode{\LL}{M}$ with $\LL_M = \{a(m''), 
    m'', a(m')\}$
    encodes $\LL$, i.e., $\MMstar(\LL_M) = \LL$.
\end{example}

\subsection{Computational Problems}

In the context of this work, we are interested in language encodings that are minimal in size.
Before we can define the computational problem of finding size-minimal macro systems, 
we define the notion of their \textit{size}.

\begin{definition}[Size of Terms, Languages, Macros Definitions, Macro Systems]
 Let $\Lang$ be a finite language, $t\in \Lang$ a term, $d= m \mapsto t$ a macro definition, \M a set of macro definitions, and $\mathfrak{M} = (\Lang_M,\macrodefinitions)$ a macro system.  Then their \emph{size} is defined as follows:
    \begin{itemize}
        \item $\size(t)$ is the number of nodes in $t$'s term tree,
        \item $\size(\Lang) := \sum_{t \in \Lang}size(t)$,
        \item $\size(d) := 1 + \size(t)$,        
        \item $\size(\macrodefinitions) := \sum_{d \in \M}size(d)$,
        \item $\size(\mathfrak{M}) := \size(\Lang_M) + \size(\mathcal{M})$.
    \end{itemize}
\end{definition}

We can now define minimal encodings  w.r.t. their  size.
\begin{definition}[Size-Minimal Encodings]
    Let $\mathfrak{M} = \fencode{\LL}{M}$ be an encoding of a language $\LL$. Then $\mathfrak{M}$ is called 
         \emph{size-minimal encoding of $\LL$ w.r.t.~\M } if there is no encoding $\mathfrak{M}'=(\LL_{M}',\M)$ of $\Lang$ with $\size(\mathfrak{M}')< \size(\mathfrak{M})$.
\end{definition}

To avoid considering macro systems that can be trivially   compressed, i.e., are not size-minimal, we will focus on sets of macro definitions that do not have duplicate macros. 

\begin{definition}[Reduced Sets of Macro Definitions]
    A set of macro definitions is \emph{reduced} if $\M^*$ is injective, i.e., there are no $m, m'\in M$ with $m\neq m'$ and $\M^*(m)=\M^*(m')$. 
\end{definition}

In what follows, we assume that all sets of macro definitions are reduced. Similarly, we will focus on encodings in which no macro can be instantiated anymore. In other words, all macros are \textit{exhaustively} instantiated.
\begin{definition}[Exhaustiveness of Encodings]
    Let $\mathfrak{M} = \fencode{\LL}{M}$ be an encoding of a language $\LL$. Then $\mathfrak{M}$ is called 
    \begin{itemize}
        \item  \emph{macrofication-exhaustive} if no macro  in $\M$ can be instantiated in $\Lang_M$,
        \item \emph{expansion-exhaustive} if no macro  in $\M$ can be instantiated in the expansion of another macro in \M,
        \item \emph{exhaustive} if it is macrofication- and expansion-exhaustive,
    \end{itemize}
\end{definition}

The problem of finding a minimal encoding of a language can be formulated
w.r.t.\ a given set of macro definitions,
an equivalence class of macro definitions,
or 
any set of macro definitions.
We provide definitions for each of these cases. 

First, we formulate the problem of finding a minimal encoding for a given language
w.r.t. given macro definitions.

\begin{problem}
\label{prob:definitions}
    Given a finite language  $\Lang \subseteq T(\Sigma)$ and \M a set of macro definitions, determine a  size-minimal encoding of $\Lang$ w.r.t.\ \M. 
\end{problem}

In Example~\ref{ex: 1}, $\encode = \fencode{\LL}{M}$ with $\LL_M = \{a(m''), 
    m'', a(m')\}$
    is a size-minimal encoding of $\LL$ with $\size((\LL_M, \MM)) = \size(\LL_M) + \size(\MM) = 5 + 11 = 16$. It is not hard to see, however, that $\encode$ can be further compressed by instantiating $m'$ in $\M(m'')$ without changing the expansion of $m''$, i.e., without really changing $\M$. 
Hence next, we first define an \emph{equivalence} of sets of macro definitions to capture this notion of "not really changing" and then use it to define a more general computational problem.

\begin{definition}[Equivalence of Macro Definitions]\label{def:macroDefinitionEquivalence}
    Two sets of macro definitions $\mathcal{M} \colon M \rightarrow T(\Sigma \cup M)$ and $\mathcal{M}' \colon M' \rightarrow T(\Sigma \cup M')$ are \emph{equivalent w.r.t.\ macro expansion} if
there exists a bijection $\mathcal{R} \colon M \leftrightarrow M'$ s.t.\ $\mathcal{M}^*(m) = \mathcal{M}'^{*}(\mathcal{R}(m))$ for all $m\in M$.
\end{definition}

\begin{problem}[Size-Minimal Encoding w.r.t.\ equivalent macro definitions]\label{prob:equivalentDefinitions}
    Given a finite language  $\Lang \subseteq T(\Sigma)$ and \M a set of macro definitions, determine a  size-minimal encoding $\mathfrak{M}' = (\Lang_{M},\mathcal{M}')$ of $\Lang$ w.r.t.\ macro definitions equivalent to \M. That is,  
    $\mathcal{M}'$ is equivalent to $\mathcal{M}$ and there does not exist an encoding
    $\mathfrak{M}'' = (\Lang_{M'}',\mathcal{M}'')$ of $\Lang$ where $\mathcal{M}''$ is also equivalent to $\mathcal{M}$
    and $\size(\mathfrak{M}'') < \size(\mathfrak{M}')$.
\end{problem}

Continuing with Example~\ref{ex: 1},  $\encode' = \fencode{\LL}{M'}$ with $\LL_{M'} = \{a(m''), 
    m'', a(m')\}$ and $\MM' = \{m\mapsto c(e), m'\mapsto d(f), m''\mapsto b(m,m')\}$ is a size-minimal encoding of $\LL$ w.r.t.~macro definitions equivalent to $\MM$: in addition to our observations regarding the minimality of $\LL_M$ above, we note that $\size(\MM')
    = \size(\LL_M) + \size(\MM') = 5 + 10 = 15 < \size(\MM) =16$,  that $\MM'$ and $\MM$ are equivalent w.r.t.~macro expansions, and that there is no equivalent set of macro definitions smaller than $\MM'$.

Lastly, we formulate the general version of the problem by looking for a minimal encoding of an input language w.r.t.\ \emph{any possible set} of macro definitions (following Definition~\ref{def:macro}).
Here we are looking for the size-minimal encoding of an input language 
considering both the size of the encoded language and the size of the macro definitions. 
 
\begin{problem}[Size-Minimal Encoding via macro systems]\label{prob:macroSystems}
    Given a finite language  $\Lang \subseteq T(\Sigma)$, determine a  size-minimal encoding $\mathfrak{M} = (\Lang_{M},\mathcal{M})$ of $\Lang$. That is,  
    there exists no encoding $\mathfrak{M}'$ of $\Lang$
    with $\size(\mathfrak{M}') < \size(\mathfrak{M})$.
\end{problem}

Continuing with Example~\ref{ex: 1}, the macrofication $\LL_{M''} = \{a(m''),   m'', a(d(f))\}$ together with the macro definitions $\MM'' = \{m'' \mapsto b(c(e),d(f))\}$ is a size-minimal encoding of $\LL$ via macros systems. The size of the encoding is $\size((\LL_M, \MM')) = \size(\LL_{M''}) + \size(\MM'') = 6 + 6 = 12$. Please note that $\MM''$ contains no macro capturing the repeated occurrence of the ``small'' term $d(f)$ since it does not pay off as it only occurs once outside $\MM''(m)$ and, for size minimality, we also consider the size of macro definitions.  
\section{Properties of Macros and Encodings}

In the following, we restrict our attention to \emph{finite} languages $\Lang$ and reduced, finite macro definitions $\mathcal{M}$.

When trying to compute a minimal macrofication for an input language w.r.t. given set of macro definitions, 
the instantiation of one macro can affect the instantiate-ability  of another macro:  instantiating $m'\mapsto d(f)$ in the first term $a(b(c(e), d(f))) \in \Lang$ in Example~\ref{ex: 1} results in $t' = a(b(c(e), m')$, in which $m''$ can no longer be instantiated. 
That is, $m'$ and $m''$  \textit{depend} on one another: since the instantiation of a macro ultimately corresponds to the replacement of its fixed-point expansion, we consider two macros to be mutually dependent if an exhaustive encoding w.r.t. their fixed-point expansions is \textit{not uniquely} determined.

\begin{definition}[Macro Dependency]
    Let $\mathcal{M} \colon M \rightarrow T(\Sigma \cup M)$ be a set of macro definitions.
    Two macros $m$ and $m'$ are \emph{independent} if a macrofication-exhaustive encoding of any finite language $\Lang \subseteq T(\Sigma)$  w.r.t.\ $\{ m \mapsto \fixedpoint{m}, m' \mapsto \fixedpoint{m'} \}$ is uniquely determined. Otherwise, the two macros are mutually \emph{dependent} on one another.
\end{definition}

Determining the dependency relationship between two macros can be reduced to the subterm relationship between their fixed-point expansions,  in which case we say that one macro \textit{contains} the other. If we focus on reduced sets of macro definitions, the "=" case is immaterial and so ignored. 

\begin{definition}[Macro Containment]
    Let $\macrodefinitions \colon M \rightarrow T(\Sigma \cup M)$ be a set of macro definitions with $m, m'\in \M$ and $m\neq m'$.
    Then $m$ \emph{contains} $m'$, if $\mathcal{M}^*(m') \preceq \mathcal{M}^*(m)$.
    By abuse of notation, we will write $m \preceq m'$ if  $m$ is contained in  $m'$.
\end{definition}

\begin{lemma}[Macro Dependency is Macro Containment]\label{lem:dependencyIsContainment}
Let $\M \colon M \rightarrow T(\Sigma \cup M)$ be a set of macro definitions. Two macros $m,m' \in M$ are dependent iff $m \preceq m'$ or $m' \preceq m$.
\end{lemma}

\begin{proof} Consider two macros $m,m' \in M$ with $\fixedpoint{m} = t$ and  $\fixedpoint{m'}= t'$.

``If'' direction: assume w.l.o.g. that $m \preceq m'$ and set $\M_{\downarrow} = \{m \mapsto t, m' \mapsto t'\}$. Consider $\Lang = \{t'\}$
and a position $p$ in $t'$ s.t. $t'|_p = t$.
So, $(\{m'\}, \M_{\downarrow})$ and $(t'[m]_p, \M_{\downarrow})$ are different macrofication-exhaustive encodings of $\Lang$ since $m' \neq t'[m]_p$.

``Only if'' direction: We show the contrapositive.
Assume that neither $m \preceq m'$ nor $m' \preceq m$, i.e., there exists no position $p$ in $t$ such that $t|_p = t'$ and no position $p$ in $t'$ such that $t'|_p = t$. Hence, instantiating one of them in a term cannot affect the instantiate-ability of the other, and thus  
$m$ and $m'$ are not dependent.
\end{proof}

\begin{lemma}
    Let $\mathfrak{M} = \fencode{\LL}{M}$ be an encoding of a language $\LL$. Then if $\mathfrak{M}$ is  
 a size-minimal encoding of $\LL$ w.r.t.~\M, then it is macrofication-exhaustive;
\end{lemma}

\begin{proof}
    This follows immediately from the definitions of size and ``can be instantiated'', and from the fact that $\size(\M^n(m))\geq 2$ for all $n\geq 1$. 
\end{proof}
\section{Algorithms for Encodings}\label{sec:algos}

\subsection{Polynomial Time Algorithm for Solving Problem 1}

We start with an algorithm that constructs a size-minimal encoding w.r.t.~a given set of macro definitions in polynomial time (i.e., solves Problem~\ref{prob:definitions}), which will  lay the foundation for solving Computational Problem~\ref{prob:equivalentDefinitions} and ultimately Computational Problem~\ref{prob:macroSystems}.
For additional proofs and more details on the three problems please refer to the Appendix Section~\ref{appsec:A}. 

The following lemma shows that, for two dependant macros, the containing one should be instantiated with priority to get a size-minimal encoding.
Note that contained macros might still be instantiated in a size-minimal encoding if they appear outside of the containing one.

\begin{lemma}[Instantiation Precedence for Dependent Macros]\label{lemma:macroPrecedence}
    Let $\macrosystem = (\Lang_M, \macrodefinitions)$  be a macro system with $\macrodefinitions \colon M \rightarrow T(\Sigma \cup M)$ 
    that encodes a language $\Lang \subseteq T(\Sigma)$.
    If there exist two 
    macros $m,m' \in M$ with $m \prec m'$ s.t.\ 
    \begin{itemize}
        \item[(a)] $m$ is instantiated instead of $m'$ in $\Lang_M$, or
        \item[(b)] $m$ is instantiated instead of $m'$ in some macro definition $m_d \mapsto t \in \mathcal{M}$,
    \end{itemize}    
    then there exists a macro system $\mathfrak{M}'$ that encodes $\Lang$ with $\size(\mathfrak{M}') < \size(\mathfrak{M})$.
\end{lemma}

\begin{proof}
Set $\fixedpoint{m}=t$ and $\fixedpoint{m'}=t'$.
By definition, $m \prec m'$ implies $t \prec t'$, i.e., $t$ occurs in $t'$ at some position $p$. Assume (a). Then, there exists a term $u\in \Lang_M$ s.t. $t'[m]_p$ occurs at position $q$ in $u$. Since $\fixedpoint{t'[m]_p} = \fixedpoint{m'}$ and $t'[m]_p$ is not a constant term, we have $\size(t'[m]_p) > \size(m')$. So, we can construct $\Lang_M'$
by replacing $u$ in $\Lang_M$ with $u[m']_q$ and construct $\mathfrak{M}' = (\Lang_M', \mathcal{M})$ with $\size(\mathfrak{M}') < \size(\mathfrak{M})$ as required.
The same argument can be made for (b) and $\M$, if $u$ is the expansion of a macro rather than a term in $\Lang_M$.
\end{proof}

 The next result shows that, from a set of pairwise independent macros,  \emph{all} should be instantiated and exhaustively. 
 
\begin{lemma}[Instantiation of Independent Macros]\label{lemma:independentMacroInstantiation}
Let $\Lang$ be a language and $\M \colon M \rightarrow T(\Sigma \cup M)$ a (reduced) set of macro definitions s.t.~any two macros in $M$ are independent. Furthermore, let $\mathfrak{M} = (\Lang_M,\M)$ be an encoding of $\Lang$ 
s.t. for all $t \in \Lang$ and all positions $p$ with $t|_p = \fixedpoint{m}$, we have $t[m]_p \in \Lang_m$.
Then there is no encoding $\mathfrak{M}'$ of $\Lang$ w.r.t. $\M$ s.t. $\size(\mathfrak{M}') <size(\mathfrak{M})$.
\end{lemma}

\begin{proof}
Assume that such a smaller $\mathfrak{M}' = (\Lang_M',\M)$ exists. Then there exists a macro $m \in M$ that is instantiated in some term $t \in \Lang$ at position $p$ in $\Lang_M'$ but not in $\Lang_M$. Since there is no $m' \in M$ s.t. $m$ and $m'$ are dependent, there is no macro that instantiates a subterm of $t|_p$ in $\Lang_M$. This means that $t|_p$ occurs in $\Lang_M$. However, this is a contradiction to 
our assumptions about $\Lang_M$ since $\fixedpoint{m} = t|_p$.
\end{proof}

\begin{theorem}\label{thm: uniqueness-and-computability-of-min-encodings}
    A size-minimal encoding of a finite language $\Lang\subseteq T(\Sigma)$ w.r.t. a finite set of (reduced) macro definitions $\mathcal{M} \colon M \rightarrow T(\Sigma \cup M)$ 
    exists, is unique, and can be constructed in polynomial time w.r.t.\ the size of $\Lang$ and $\mathcal{M}$.
\end{theorem}

\begin{proof}
The existence of a minimal encoding is trivial. Since both $\Lang$ and $\mathcal{M}$ are finite,
all possible encodings can be enumerated and compared w.r.t.\ their size.

Next, we sketch out a simple algorithm to compute a minimal encoding of $\Lang$ w.r.t. $\M$ and show that each step preserves uniqueness and is polynomial: 
\begin{enumerate}
    \item 
   \begin{enumerate}
       \item compute $\fixedpoint{m}$ for each $m\in M$
       \item determine all dependencies between  macros in \M, and 
       \item compute the associate Hasse-Diagram of macro dependencies $\prec$.
   \end{enumerate}  This can be constructed by a naive algorithm in 
    $O(|\mathcal{M}|^2* \mu^2)$ 
    steps, where 
    $\mu$ is the maximum $\size(\fixedpoint{m})$ of $m\in M$ since subtree isomorphism of labelled, unordered trees can be decided in quadratic time \cite{Valiente2002}.
    We can further reduce this to $O(|\mathcal{M}|^2* \tau^2)$ 
    steps, where 
    $\tau$ is the maximum $\size(t)$ of $t\in \LL$ since any macro with extension larger than $\tau$ can never be instantiated and thus be ignored.
    \item Starting from the top, traverse the Hasse-Diagram of macro dependencies in a breadth-first manner (starting with maximally containing macros) and exhaustively instantiate all macros $m$ on  each level in $\Lang$, i.e., compute $\Lang_{S\cup\{m\}} := 
    \{t' \mid t' \text{ is obtained by replacing all occurrences of }\fixedpoint{m} \text{ in }$ $t\in \Lang_S \text{ with } m\}$ 
    starting from $\LL_\emptyset = \LL$.

    Since the macros on each level are independent (due to \M being reduced) and since we exhaustively instantiate a macro before we possibly instantiate macros it contains, the resulting macro system $\mathfrak{M}= (\Lang_{M}, \M_M)$ is unique and size-minimal; see Lemmas~\ref{lemma:independentMacroInstantiation} and~\ref{lemma:macroPrecedence}.

    Each of these $|\mathcal{M}|$ steps for macros $m$ can be carried out in $|\mathcal{L}|*\tau ^2$: 
    we need to test all terms in  $\mathcal{L}$ whether they contain $\fixedpoint{m}$ which can be assumed to have size $\leq \tau$. 
\end{enumerate}
Thus, the worst-case runtime of the algorithm described abode is $O(|\mathcal{M}|^2 \cdot \tau^2 + |\mathcal{M}| \cdot |\Lang| \cdot \tau^2)$.
\end{proof}
Generalising this to non-reduced macro systems is straightforward and we  only lose uniqueness.

\subsection{Polynomial Time Algorithm for Solving Problem~2} 

\begin{theorem}\label{thm: uniqueness-and-computability-of-min-encodings-problem-2}
    A size-minimal encoding of a finite language $\Lang\subseteq T(\Sigma)$ w.r.t. the class of equivalent macro definitions of a finite set of (reduced) macro definitions $\mathcal{M} \colon M \rightarrow T(\Sigma \cup M)$ 
    exists, is unique up to macro renaming, and can be constructed in polynomial time w.r.t.\ the size of $\Lang$ and $\mathcal{M}$.
\end{theorem}

\begin{proof}
    Define the language $\LL'$ to be the union of terms in $\LL$ and in $\MM$, i.e., $\LL' = \LL \cup \{\fixedpoint{t} \mid m \mapsto t \in \MM\}$.
    Then the size of $\LL'$ is bounded by $2 \cdot size(\LL) \cdot size(t_{max})$, where $t_{max}$ is the largest term in $\LL$.
    So, a size-minimal encoding of $\LL'$ with respect to $\MM$ is unique and can be computed in polynomial time according to Theorem~\ref{thm: uniqueness-and-computability-of-min-encodings}.
    Let $\encode' = \fencode{\LL'}{M}$ be such an encoding. Then the terms $t' \in \LL'_M$ that correspond to terms of macro definitions have been compressed to be size minimal. 
    Let $t'_m\in \LL'_M$ be the term that corresponds to the expansion of $m$, i.e., $\MMstar(m) = \MMstar(t'_m)$, and let 
     $\LL_{M'}$ be the set of remaining terms in $\LL'_M$.
    Then  $\MM'=\{m \mapsto t'_m \mid m\in M\}$ is equivalent to $\MM$ and  $\encode = (\LL_{M'}, \MM')$ is a size-minimal encoding w.r.t. macro definitions equivalent to $\MM$.
    Note that as a consequence of Theorem~\ref{thm: uniqueness-and-computability-of-min-encodings}, the encoding $\encode$ is unique up to permutations over the macro symbols $M$. 

        It is straightforward to extend the algorithm described in the proof of Theorem~\ref{thm: uniqueness-and-computability-of-min-encodings} to also maximally compress $\MM$: in Step~2, in addition to computing $\Lang_{S\cup\{m\}}$ for each $m\in M$, also compute 
      $\MM_{S\cup\{m\}} := 
    \{m'' \mapsto t' \mid t' \text{ is obtained by replacing all occurrences of }\fixedpoint{m} \text{ in }t $ $\text{with } m,\text{ for any }m''\mapsto t \in \MM_S \}$, with $\MM_\emptyset = \MM$.
\end{proof}

Note that in the previous result, uniqueness only holds up to renaming of macros, because equivalence of macro definitions allows this.
As previously remarked for Problem~\ref{prob:definitions}, uniqueness (up to renaming) stems from the reducedness of the macro definitions. We may allow macro definitions to be non-reduced but must then choose which of two equivalent macro definitions to instantiate at any applicable position.

\subsection{Polynomial Time Algorithm for Solving Problem~\ref{prob:macroSystems}}

Recall that the minimal encoding of a language with equivalent macro definitions 
is unique (up to renaming) by Theorem~\ref{thm: uniqueness-and-computability-of-min-encodings-problem-2}. 
For the set of reduced macro definitions $\MM$ and language $\LL$, we denote the minimal encoding without renaming macro symbols by $\minencode{\LL}{\MM}$.
In Problem~\ref{prob:macroSystems}, we are thus interested in determining a set of macro definitions $\MM_{\textit{opt}} \in \argmin_{\MM} \size(\minencode{\LL}{\MM})$ for a given language $\LL$.

To determine which macro definitions to include in such a minimal encoding, it will be elementary to consider occurrences and sizes of subterms in encodings.

\begin{definition}[Occurrences \& Dominance of terms]
    Consider an encoding $\encode = \fencode{\LL}{M}$ of  $\LL\subseteq T(\Sigma)$.
    The \emph{number of occurrences} of a term $t \in \mterms{M}$ in $\encode$ 
    is defined as follows: $\occ{\encode}(t) := $
$$
        \sum_{t' \in \LL_M} |\{p \mid t'|_p = t\}| 
        +\sum_{m' \in M} |\{p \mid \MM(m)|_p = t\}|.
  $$
    We write $\occ{\LL}(t)$ as a short form of $\occ{(\LL,\emptyset)}(t)$, and
    say a term $t$ \emph{is dominated} by a term $t'$ in $\LL$ if $t \preceq t'$ and $\occ{\LL}(t) = \occ{\LL}(t')$.
    In this case, we also say $t'$ \emph{dominates} $t$.
\end{definition}

If $t$ is dominated by $t'$ in $\LL$, then $t$ only occurs within $t'$ in the language $\LL$ and also only occurs exactly once in $t'$. It is not hard to show (see Lemma~\ref{le: lets-not-include-dominated-terms}) that $\MM_{\textit{opt}}$ contains no macro definition $m\mapsto t$ where $\MM_{\textit{opt}}^*(t)$ is dominated by a term $t'\neq \MM_{\textit{opt}}^*(t)$ in $\LL$: 
otherwise the encoding could be further compressed by removing  $m\mapsto t$ and adding $m\mapsto t'$.

Furthermore,
$\MM_{\textit{opt}}$ contains only macro definitions $m\mapsto t$ where $\occ{\Lang}(\MM_{\textit{opt}}^*(t)) \geq 2$: introducing macros for terms that occur once increases the size of the encoding by $1$.

Finally, in the absence of unary symbols in $\Sigma$, Lemmas~\ref{le:occurences-do-not-go-below-2}, \ref{le:sizes-do-not-go-below-3}, and \ref{le: size-in-or-decreases-from-macro-def} together show that $\MM_{\textit{opt}}$ contains a macro for \textit{each} (sub-)term 
$t'\in \Lang$ with $\occ{\Lang}(t)\geq 2$ that is \emph{not} dominated by any other subterm in $\Lang$.\footnote{In the presence of unary symbols in $\Sigma$, this simple criterion fails to work; see the Appendix for an example.}

This results in a characterisation of the set of macro definitions for a minimal encoding of a language
that can be constructed in polynomial time.

\begin{theorem}\label{thm: uniqueness-and-computability-of-min-encodings-problem-3}
    For a finite language $\Lang \subseteq \terms$ where $\Sigma$ does not include unary symbols, a size-minimal encoding of $\LL$
    is given by
    $\minencode{\LL}{\MM}$ w.r.t. macro definitions 
 \[\MM := 
   \{m_t \mapsto t \mid \begin{array}[t]{@{\,}l}  t \text{ subterm in } \LL  \text{ with } \occ{\LL}(t) \geq 2  \text{ and}\\ \nexists \text{ subterm }t' \text{ in } \LL \text{ dominating $t$ in } \LL\}\end{array}\]
    and can be constructed in  time polynomial in  $\size(\Lang)$.
\end{theorem}

\begin{proof}
    This is a direct consequence of Lemmas~\ref{le:occurences-do-not-go-below-2} - 
    \ref{le: lets-not-include-dominated-terms}, plus the fact that occurrence and dominance of subterms in $\Lang$ can be computed in time polynomial in $\size(\Lang)$.
\end{proof}

The algorithm sketched in the proofs of Theorems~\ref{thm: uniqueness-and-computability-of-min-encodings} and \ref{thm: uniqueness-and-computability-of-min-encodings-problem-2} can be easily extended to one for Problem~3 by constructing $\MM$ prior to the 1st step: considering all nodes $n$ in all term trees $t\in \Lang$,  determine how often $t|_n$ occurs in $\LL$ and, if this is more than once, whether $t_n$  is dominated, i.e., contained by other subterms with equal number of occurrences.

\section{Empirical Evaluation}

\begin{table*}[t]
\footnotesize
  \centering
   \caption{Statistics on Minimising Ontologies showing the size of the ontology $\ont$, its macrofication $\Lang_M$, and its macro definitions $\M$ (in the case of Problem 1 and 2, definitions are measured by the size of associated axioms), 
  the proportional size reduction, Prop. Red. (measured by $(1- (\size(\Lang_M) + \size(\M))/\size(\ont))*100$), the number of axioms of the input language (\# Axioms), the number of changed axioms (\# Ch. Axioms), and the number of macro definitions (\# $\M$).}
  \begin{adjustbox}{width=0.74\textwidth,center}

  \begin{tabular}{| c | c | r | r |  r | r | r | r | r | }
      \hline

P & O & $\size(\mathcal{O})$ & $\size(\mathcal{L}_M)$ & $\size(\mathcal{M})$ &  Prop. Red. &  \# Axioms & \# Ch. Axioms &  \# $\M$ ~~~\\
\hline
\hline
1 &         &                        & \numprint{1291740}  & \numprint{2269878}    & 3\%   &  \numprint{227909}  & \numprint{9199}    & \numprint{132588}  \\
2 & S       & \numprint{3657845}                & \numprint{1291740}  & \numprint{2269878}    & 3\%   &  \numprint{227909}  & \numprint{9199}    & \numprint{132588}  \\
3 &         &                        & \numprint{2013130}  & \numprint{293779}     & 37\%  &  \numprint{360497}  & \numprint{188066}  & \numprint{54883}  \\
\hline                                 
1 &         &                        & \numprint{104752}   & \numprint{90435}      & 2\%   &  \numprint{25563}   & \numprint{686}     & \numprint{9968} \\
2 &  G      & \numprint{198568}                 & \numprint{104752}   & \numprint{87961}      & 3\%   &  \numprint{25563}   & \numprint{1152}    & \numprint{9968} \\
3 &         &                        & \numprint{151442}   & \numprint{16496}      & 15\%  &  \numprint{35531}   & \numprint{14433}   & \numprint{3624} \\
\hline                                
1 &         &                        & \numprint{31905}    & \numprint{2627}       & 2\%   &  \numprint{3992}    & \numprint{177}     & \numprint{330} \\
2 & C       & \numprint{35361}                  & \numprint{31905}    & \numprint{2528}       & 3\%   &  \numprint{3992}    & \numprint{210}     & \numprint{330} \\
3 &         &                        & \numprint{20501}    & \numprint{2340}       & 35\%  &  \numprint{4322}    & \numprint{585}     & \numprint{439} \\
\hline                               
1 &         &                        & \numprint{138740}   & \numprint{90886}      & 21\%  &  \numprint{37481}   & \numprint{5524}    & \numprint{7550} \\
2 &  F      & \numprint{289643}                 & \numprint{138740}   & \numprint{90874}      & 21\%  &  \numprint{37481}   & \numprint{5525}    & \numprint{7550} \\
3 &         &                        & \numprint{144021}   & \numprint{64298}      & 28\%  &  \numprint{45030}   & \numprint{18650}   & \numprint{9814} \\
\hline      
1 &         &                        & \numprint{830107}   & \numprint{319743}     & 0\%   &  \numprint{237060}  & \numprint{0}       & \numprint{19069} \\
2 &  N      & \numprint{1149850}                & \numprint{830107}   & \numprint{319743}     & 0\%   &  \numprint{237060}  & \numprint{0}       & \numprint{19069} \\
3 &         &                        & \numprint{913719}   & \numprint{36158}      & 17\%  &  \numprint{256129}  & \numprint{60752}   & \numprint{9019} \\
  \hline

  \end{tabular}
  \end{adjustbox}

  \label{tab:minimizing-ontologies}
\end{table*}

We investigate whether the 3 algorithms presented in Section~\ref{sec:algos} can be implemented to efficiently and effectively rewrite large OWL ontologies of practical relevance.

\paragraph{Implementation:} We implement the three algorithms
described in Section~\ref{sec:algos} for OWL ontologies in Java (17 LTS) using the OWL API (5.1.15) and JGraphT (1.5) to handle ontologies and trees, respectively. All experiments were run on an Apple M2
Max processing unit with 96GB of RAM under the operating system macOS Ventura (13.5). All source code is available at \url{https://github.com/ckindermann/AAAI2024}.

\paragraph{Corpus:} 
We use 5 large ontologies that were originally selected for the evaluation of a related approach, cf.~\cite{DBLP:conf/aaai/NikitinaK17}. The latest version of each ontology was downloaded from BioPortal\footnote{See \url{https://bioportal.bioontology.org/}} on August 8, 2023. In the case of SNOMED the most recent release of February 2023 was used.
The size of each ontology
is shown
in the third column of   Table~\ref{tab:minimizing-ontologies}, where S stands for SNOMED, G for Galen, C for Genomic CDS, F for FYPO, and  N for NCIT.

\paragraph{Setup:}
An axiom of the form $\opr{EquivalentClasses}(\conc{N},\conc{C})$ can be seen as macro definition of the form $\conc{N} \mapsto \conc{C}$ if $\conc{N}$ is a named class and $\conc{C}$ is a complex class expression.
To apply our approach, however, we require macro definitions to be functions and acyclic (see Definition~\ref{def:acyclic}). 
For an ontology $\ont$, let $\M_\ont \subseteq \ont$ denote the set of macro definitions obtained from  $\ont$'s  $\opr{EquivalentClasses}(\conc{N},\conc{C})$ axioms without axioms  that (a) are  involved in cycles or (b) in which $\conc{N}$ occurs more than once in this set of axioms.
Note that we do not reduce $\M_\ont$ because we are interested in the impact of \textit{macro instantiation} on the size of an encoding of an ontology rather than the impact of \textit{removing redundant} definitions. 

Expanding all macros exhaustively in $\ont \setminus \M_\ont$ yields the language $\Lang = \M_\ont^{*}(\ont \setminus \M_\ont)$ without macro symbols and without their definitions. By abuse of notation, $(\Lang, \M_\ont)$ can be seen as an encoding of $\ont$ and can be used as an input macro system for Problems 1 and 2. For Problem 3, we use $\ont$ as the input language but remove all axioms involving unary symbols, i.e., \opr{ComplementOf}.\footnote{Which caused the removal of only 1 axiom in 1 ontology.}

\paragraph{Reduction Results:}
In Table~\ref{tab:minimizing-ontologies}, we list statistics for the size-minimal encoding for each pair of the 3 computational problems and 5 ontologies.
We can see that size-minimal encodings w.r.t. Problem 3 lead to significantly larger size-reductions compared to Problem 1 and 2, suggesting that the introduction of \textit{new macro symbols} can capture and reveal structural patterns in ontologies that cannot be encoded via existing named classes defined via \opr{EquivalentClasses} axioms, i.e., macros in Problem 1 and 2.

A comparison with the size reduction of class expressions reported for a related approach, cf.~\cite{DBLP:conf/aaai/NikitinaK17}, reveals
that encodings w.r.t. Problem 3 often yield comparable reductions while reducing larger proportions of class expressions (see Table~\ref{tab:minimizing-expressions} in the Appendix).

\paragraph{Processing Time:} We can compute each size-minimal encoding for each ontologies w.r.t.~any of the three problems in less than a minute. This is a significant improvement over existing approaches, cf.~\cite{ DBLP:conf/aaai/NikitinaK17},
where minimising 100 concepts in SNOMED was reported to take on average 3--10 minutes depending on concept size.
\section{Related Work}\label{sec:relatedWork}

The notion of a ``macrofication'', i.e., the idea of refactoring source code by introducing macros in an automated manner has already been proposed~\cite{DBLP:conf/esop/SchusterDF16}. The motivation for
reversing macro expansion, or what we call macro instantiation, is improved readability and maintenance. However, the objective of smaller rewritings is subjected to the condition of not altering program behaviour rather than aiming for size-minimality. 

Automated rewriting mechanisms for OWL ontologies, designed to aide ontology comprehension and maintenance include techniques for removing redundancies based on entailment~\cite{DBLP:conf/esws/GrimmW11}, removing redundancies based on templates~\cite{DBLP:conf/semweb/SkjaevelandLKF18}, reusing equivalent named classes~\cite{DBLP:conf/semweb/KindermannS22a}, repairing unsatisfiable classes~\cite{DBLP:conf/esws/KalyanpurPSG06},  and adding missing but desired as well as removing undesired relationships~\cite{DBLP:conf/mie/Mougin15}. While all of these approaches reduce an ontology's size, none of the proposed techniques aim for size-minimality.

Existing rewriting approaches for OWL ontologies, aiming for some notion of size-minimality are based on entailment~\cite{DBLP:conf/kr/BaaderKM00,DBLP:conf/aaai/NikitinaK17}. However, these approaches only work for non-expressive Description Logics, namely $\mathcal{ALE}$ and $\mathcal{EL}$, and must be restricted further, e.g., to exclude the important OWL constructor \opr{IntersectionOf} \cite{DBLP:conf/semweb/NikitinaS13} or to require a notion of semantic acyclity~\cite{DBLP:conf/dlog/KoopmannN16} to mitigate unfavourable computational characteristics, limiting their application in practice.

 Besides earlier versions of this work, cf.~\cite{kindermann2021analysing},
this is the first rewriting study that considers (un)ranked and mixed symbols and is thus truly applicable to complex languages such as OWL.
\section{Conclusions and Outlook}

In this paper, we present a rewriting mechanism for finite formal languages based on (nested) syntactic macros.
Using this mechanism, we demonstrate the feasibility and effectiveness of computing size-minimal rewritings for large ontologies of practical relevance in the biomedical domain.

The presented work can be extended in a variety of ways. The complexity of the general Problem 3, for languages over alphabets including unary symbols, remains an open question. 
Moreover, different features of macro systems can be investigated, e.g., 
\textit{1-step macros} (no nesting), or
\textit{parameterised macros} of the form $m(x_1,x_2) \mapsto a(x_1,b(c,x_2))$, with variables 
that can be bound to input arguments.

Another interesting direction for future work are user studies to investigate the usefulness of size-minimal encodings in potential use-case scenarios. Such studies may shed light on what kind of interaction mechanisms are required for assisting with refactoring tasks in practice.

\clearpage
\appendix
\section{Supplementary Material}\label{appsec:A}

\subsection{Details for Section \ref{sec:algos}}\label{A:details}

\subsubsection{Details for Theorem~\ref{thm: uniqueness-and-computability-of-min-encodings-problem-3}}
In the following, we give details for the results necessary to proof Theorem~\ref{thm: uniqueness-and-computability-of-min-encodings-problem-3}.

We can show that a macrofied version of any term $t$ that occurs at least twice in a language and is not dominated also occurs at least twice in any minimal encoding of the language.

\begin{lemma}\label{le:occurences-do-not-go-below-2}
    Consider a (sub-)term $t=t'|_p$ of $t' \in \LL \subseteq \terms$ such that $\occ{\LL}(t) \geq 2$,  let $\MM \colon M \rightarrow \mterms{M}$ be a macro system,
    and let $t_{r}$ be the macrofied version of $t$ in $\minencode{\LL}{\MM}$. 
    If $t$ is not dominated by any other (sub-)term $t''|_q$ of $t'' \in \LL$, then $\occ{\minencode{\LL}{\MM}}(t_{r}) \geq 2$.
\end{lemma}
\begin{proof}
    Consider the setting  described in Lemma~\ref{le:occurences-do-not-go-below-2}. At each occurrence of $t$ in $\LL$, $t$ is either eliminated in $\minencode{\LL}{\MM}$ by instantiating a macro $m \mapsto t_m$ with $t \prec \MMstar(t_m)$ at a higher position, or we have an occurrence of $t_{r}$.
    
 In the former case, $t_r$ occurs in a macro definition and thus at least once in $\minencode{\LL}{\MM}$.
    Assume further $\occ{\minencode{\LL}{\MM}}(t_r) = 1$. Hence  $t_r$ occurs in a single macro definition and only once. 
    Assume that such a macro definition $m \mapsto t_m$ exists in $\minencode{\LL}{\MM}$.
    Then $\MMstar(t_m)$ dominates $t$, contradicting our choice of $t$. Hence $\occ{\minencode{\LL}{\MM}}(t_r) \geq 2$.
\end{proof}

As a consequence of the uniqueness of minimal encodings stated in Theorem~\ref{thm: uniqueness-and-computability-of-min-encodings}, a macrofied version of any term $t$ always has the same shape in a minimal encoding, i.e., for a minimal encoding $\minencode{\LL}{\MM}$ of  $\LL$ w.r.t~$\MM$ and for any $t_1, t_2\in\minencode{\LL}{\MM}$ with $\MMstar(t_1) = t = \MMstar(t_2)$, we have $t_1=t_2$. 
Hence we can call $t_1$ \emph{the }macrofied version of $t$ in $\minencode{\LL}{\MM}$ and treat all occurrences of a term $t$ in the same way.

Similarly, we can show that if a language does not include unary symbols, any term will maintain a size of at least $3$ in any size-minimal encoding.

\begin{lemma}\label{le:sizes-do-not-go-below-3}
    Consider a non-constant (sub-)term $t=t'|_p$ of $t'  \in \LL \subseteq \terms$ where $\Sigma$ contains no unary symbols. Let 
   $\minencode{\LL}{\MM} = \fencode{\LL}{M}$ be  a minimal encoding  of $\LL$ with $\MM \colon M \rightarrow \mterms{M}$ and $t_r$ the macrofied version of $t$ in $\minencode{\LL}{\MM}$. 
    Then $\size(t_r) \geq 3$.
\end{lemma}
\begin{proof}Consider the setting as described in Lemma~\ref{le:sizes-do-not-go-below-3}.
    Hence $t$ is of the form $s(t_1, t_2, \ldots)$. If $t_r=t$, the claim follows trivially. If $t_r\in M$, $\size(t_r)= 1+ \size(\M(t_r))\geq 3$ because $\M(t_r)$ has at least two children. Otherwise $t_r$'s root has at least two children and $\size(t_r) \geq 3$.
\end{proof}

We can now characterise which additions of macro definitions do not increase the size of an encoding when added.

\begin{lemma}\label{le: size-in-or-decreases-from-macro-def}
Let $\LL \subseteq \terms$ be a language where $\Sigma$ contains no unary symbols, and let $\MM\colon M \longrightarrow \mterms{M}$ be a set of reduced macro definitions.
    Consider a "new" macro definition $m \mapsto t\in \terms$ with $m \notin M$ such that, for all $m' \in M$,   $\MMstar(m')$ does not dominate $t$ in $\LL$. 
     If $\occ{\LL}(t) \geq 2$
    then $\size(\minencode{\LL}{\MM}) \geq \size(\minencode{\LL}{\MM \cup \{m \mapsto t\}})$.
\end{lemma}
\begin{proof}
Consider the setting as described in Lemma~\ref{le: size-in-or-decreases-from-macro-def}. 
    Since $t$ is not a constant and $\Sigma$  includes no unary symbols, we have $\size(t) \geq 3$.
    By Lemma~\ref{le:sizes-do-not-go-below-3} it follows that for the minimal encoding $\minencode{\LL}{\MM}$ of $\LL$ the macrofied version $t_r$ of $t$ also has size at least 3.
    By Lemma~\ref{le:occurences-do-not-go-below-2} and because $t$ is not dominated in $\LL$,  $t_r$ occurs at least twice in $\minencode{\LL}{\MM}$ if $\occ{\LL}(t) \geq 2$.
    Comparing $\minencode{\LL}{\MM}$ and $\minencode{\LL}{\MM \cup \{m \mapsto t\}}$, we note that in the latter, the macro $m$ is instantiated at each occurrence of $t_r$ in $\minencode{\LL}{\MM}$, 
    and the macro definition $m \mapsto t_r$ is part of macro definitions in the latter. Hence 
 $\occ{\LL}(t) \geq 2$ implies that 
    \begin{align*}
        & \size(\minencode{\LL}{\MM \cup \{m \mapsto t\}})\\
        \leq & \size(\minencode{\LL}{\MM}) - \occ{\minencode{\LL}{\MM}}(t_r) \cdot (\size(t_r) -1) \\
        &+ \size(m\mapsto t_r))\\
        = & \size(\minencode{\LL}{\MM}) - (\occ{\minencode{\LL}{\MM}}(t_r) - 1) \cdot (\size(t_r) -1) + 2\\
        \leq & \size(\minencode{\LL}{\MM}) - 1 \cdot 2 + 2\\
        \leq & \size(\minencode{\LL}{\MM}).
    \end{align*}
\end{proof}

In Lemma~\ref{le: size-in-or-decreases-from-macro-def}, we restrict our attention to languages of alphabets without unary symbols. To see why unary symbols cause problems, consider the example $\LL:=\{s(x(a(b(c))), x(a(b(c))), y(a(b(c))), y(a(b(c))))\}$. The term $b(c)$ is not dominated in $\LL$ as it occurs 4 times, twice as subterm of $x(.)$ and twice as a subterm of $y(.)$. Hence we could think that a minimal encoding of $\LL$ involves a macro $m\mapsto b(c)$. The size of the minimal encoding of $\LL$ w.r.t. $\{m\mapsto b(c)\}$ is 16, and further adding macros $ m_1 \mapsto x(a(m)),\  m_2 \mapsto y(a(m))$ results in a minimal encoding of the same size. However, the size of the minimal encoding $(\{s(m_3, m_3, m_4, m_4), \ \{m_3\mapsto x(a(b(c)),\ m_4 \mapsto y(a(b(c))\} )$ is only 15.

Lemma~\ref{le: size-in-or-decreases-from-macro-def} implies that any non-dominated term that occurs at least twice in a language (without unary symbols) is part of a macro definition in a minimal encoding.
This does not per say exclude other terms to be set up as macros as well.
The following lemma shows that, whenever a term is dominated, it is more advantageous to include the dominating term instead.

\begin{lemma}\label{le: lets-not-include-dominated-terms}
    Let $\minencode{\LL}{M}$ be a minimal encoding for language $\LL$ w.r.t. some reduced set of macro definitions $\MM$.
    Consider new macro symbols $m, m' \notin M$ and terms $t, t' \in \mterms{M}$ such that $\MMstar(t)$ is dominated by $\MMstar(t')$ in $\Lang$, $\occ{\minencode{\LL}{M}}(t) = \occ{\minencode{\LL}{M}}(t') \geq 2$ 
    and for each $s\in \{t, t'\}$ and each $m''\in M,$ $\fixedpoint{m''}\neq \fixedpoint{s}$. 
    Then $$\size(\minencode{\LL}{M \cup \{m' \mapsto t'\}}) < \size(\minencode{\LL}{M \cup \{m \mapsto t\}})$$ and 
    $$\size(\minencode{\LL}{M \cup \{m' \mapsto t'\}}) < \size(\minencode{\LL}{M \cup \{m' \mapsto t', m \mapsto t\}}).$$
\end{lemma}

\begin{proof}
    As  in the proof of Lemma~\ref{le: size-in-or-decreases-from-macro-def}, we have $\size(\minencode{\LL}{\MM \cup \{m \mapsto t\}}) = \size(\minencode{\LL}{\MM}) - (\occ{\minencode{\LL}{\MM}}(t_r) - 1) \cdot (\size(t_r) -1) + 2$ and analogously for $\size(\minencode{\LL}{\MM \cup \{m' \mapsto t'\}})$.
    Since $\occ{\minencode{\LL}{M}}(t) = \occ{\minencode{\LL}{M}}(t') \geq 2$ both $t$ and $t'$ must be
    (possibly macrofied) versions of some terms in $\LL$ such that no other macro definition in $\minencode{\LL}{M}$ can be instantiated.
    Furthermore, because $\MMstar(t)$ is dominated by $\MMstar(t')$, we have $t \prec t'$ and thus $\size(t') > \size(t)$.
    Together with $\occ{\encode}(t') = \occ{\encode}(t) \geq 2$, we have
    \begin{align*}
        & \size(\minencode{\LL}{\MM \cup \{m' \mapsto t'\}})\\
        = & \size(\minencode{\LL}{\MM}) - (\occ{\minencode{\LL}{\MM}}(t') - 1) \cdot (\size(t') -1) + 2\\
         = & \size(\minencode{\LL}{\MM}) - (\occ{\minencode{\LL}{\MM}}(t) - 1) \cdot (\size(t') -1) + 2 \\
         < & \size(\minencode{\LL}{\MM}) - (\occ{\minencode{\LL}{\MM}}(t) - 1) \cdot (\size(t) -1) + 2 \\
         = & \size(\minencode{\LL}{\MM \cup \{m \mapsto t\}}).
    \end{align*}

    Now consider including both macro definitions in a minimal encoding, $\minencode{\LL}{M \cup \{m' \mapsto t', m \mapsto t\}}$.
    Because of minimality, $t'$ only occurs in the macro definition $m' \mapsto t'$ in $\minencode{\LL}{M \cup \{m' \mapsto t'\}}$.
    Furthermore, because  $\MMstar(t)$ is dominated by $\MMstar(t')$ and thus $m \prec m'$, we can only instantiate either  $m$ or $m'$ in $\LL$ for $\minencode{\LL}{M \cup \{m' \mapsto t', m \mapsto t\}}$.
    By Lemma~\ref{lemma:macroPrecedence} we instantiate $m'$ at any occurrence of $t'$.
    Thus, the only instantiation of $m$ in $\minencode{\LL}{M \cup \{m' \mapsto t', m \mapsto t\}}$ is in the macro definition of $m'$.
    Hence $\size(\minencode{\LL}{M \cup \{m' \mapsto t', m \mapsto t\}}) = \size(\minencode{\LL}{M \cup \{m' \mapsto t'\}}) - (\size(t) - 1) + \size(t) + 1 > \size(\minencode{\LL}{M \cup \{m' \mapsto t'\}})$.
\end{proof}

\subsubsection{Polynomial Time Algorithms for Solving Problem 1 - 3: Pseudo-code}

Note our small typo in the main paper that claims to instantiate macros in $\LL$ and $\MM$ in the proof of Theorem~\ref{thm: uniqueness-and-computability-of-min-encodings}. Instead, we only instantiate macros in $\LL$ as is desired in Problem 1. This is expressed in the second for-loop of the following pseudo-code for solving Problem~1.

  \begin{algorithm}[H]
         \SetAlgoLined
         \SetKwFunction{algo}{algo}\SetKwFunction{proc}{macroIntroduction} 
                 \LinesNumbered

         \KwResult{Minimal Encoding $(\Lang_M,\mathcal{M})$ of input language $\Lang$ w.r.t.\ input definitions $\mathcal{M}$}
         $\mathcal{H} \gets$ \texttt{MacroDependHasseDiagram}($\mathcal{M}$)\;
         
         $\Lang_M \gets \Lang$\;
         
         level $\gets$ \texttt{nextLevel}($\mathcal{H}$) 
         
         \While{{\normalfont level} is not empty}{
             \For{$m \in$ {\normalfont level}}{
             
                 $m^* \gets \mathcal{M}^*(m)$
                 
                  \For{$t \in \Lang_M$}{
                     \For{all $p$ with $t|_p = m^*$} {
                         $t \gets t[m]_p$
                     }
                  }            
             }
             level $\gets$ \texttt{nextLevel}($\mathcal{H}$)\;
         }
         \Return{$(\Lang_M, \mathcal{M})$}
         \caption{Size-Minimal Encoding w.r.t. Problem 1}\label{alg:maximalConstantLanguageCompression}
     \end{algorithm}
     
The proof of Theorem~\ref{thm: uniqueness-and-computability-of-min-encodings-problem-2} constructs a minimal encoding by considering macro definitions as part of the input language so that Algorithm~\ref{alg:maximalConstantLanguageCompression} can be applied. In essence, we show that the arguments for constructing a size-minimal macrofication for a language can also be applied to constract a size-minimal set of macro definitions in the obvious way. This is outlined in the pseudo-code shown in Algorithm~\ref{alg:maximalConstantLanguageCompression-problem-2}:

\begin{itemize}
    \item lines 1--11 correspond to the construction of a minimal macrofication, i.e., Algorithm~\ref{alg:maximalConstantLanguageCompression},
    \item lines lines 12--15 instantiate macros in macro definitions in an analogous way.
\end{itemize}
 \begin{algorithm}[H]
          
         \SetAlgoLined
             \setcounter{AlgoLine}{0}

         \SetKwFunction{algo}{algo}\SetKwFunction{proc}{macroIntroduction} 
                          \LinesNumbered

         \KwResult{Minimal Encoding $(\Lang_M,\mathcal{M})$ of input language $\Lang$ w.r.t.\ input definitions $\mathcal{M}$}
         $\mathcal{H} \gets$ \texttt{MacroDependHasseDiagram}($\mathcal{M}$)\;
         
         $\Lang_M \gets \Lang$\;
         
         level $\gets$ \texttt{nextLevel}($\mathcal{H}$) 
         
         \While{{\normalfont level} is not empty}{
             \For{$m \in$ {\normalfont level}}{
             
                 $m^* \gets \mathcal{M}^*(m)$
                 
                  \For{$t \in \Lang_M$}{
                     \For{all $p$ with $t|_p = m^*$} {
                         $t\gets t[m]_p$
                     }
                  }   
                  \tcp{macro definitions}\;
                  \For{$m \mapsto t \in \M$}{
                     \For{all $p$ with $t|_p = m^*$} {
                         $(m \mapsto t) \gets (m \mapsto t[m]_p)$
                     }
                  }  
             }
             level $\gets$ \texttt{nextLevel}($\mathcal{H}$)\;
         }
         \Return{$(\Lang_M, \mathcal{M})$}
         \caption{Size-Minimal Encoding w.r.t. Problem 2}\label{alg:maximalConstantLanguageCompression-problem-2}
     \end{algorithm}

For Problem 3, we know by Theorem~\ref{thm: uniqueness-and-computability-of-min-encodings-problem-3} which macro definitions $\MM$ to consider and by Theorem~\ref{thm: uniqueness-and-computability-of-min-encodings-problem-2} that we can simply construct a size-minimal encoding with macro definitions equivalent to $\MM$, i.e., applying Algorithm~\ref{alg:maximalConstantLanguageCompression-problem-2}. 

 \begin{algorithm}[H]
         \SetAlgoLined
          \LinesNumbered
        \setcounter{AlgoLine}{0}

         \SetKwFunction{algo}{algo}\SetKwFunction{proc}{macroIntroduction} 
         \KwResult{Selection of macro definitions $\M$ suitable for solving Problem 3 (without unary symbols)}
         $\mathcal{T} \gets \{t \colon \text{ there exists } t' \in \LL \text{ with } t \preceq t' \text{ and } \occ{\LL}(t) \geq 2\}$
         
         \For{$t, t' \in \mathcal{T}$}{
            \If{$t \preceq t'$ and $\occ{\LL}(t) = \occ{\LL}(t')$}{
            $\mathcal{T} \gets \mathcal{T} \setminus \{t\}$;
            }
         }
         $\M \gets \emptyset$
         
         \For{ $t \in \mathcal{T}$}{

         $\M \gets \M \cup \{m_t \mapsto t\}$
         
         }

         \Return \M
         \caption{Seclection of Macro Definitions for Problem 3}\label{alg:maximalConstantLanguageCompression-problem-3}
     \end{algorithm}

Note again our restriction to languages over non-unary symbols. This is necessary to guarantee that a selected macro definition indeed guarantees a size reduction in combination with any other selected macro definitions. In particular, this was the necessary condition in Lemma~\ref{le:sizes-do-not-go-below-3} for the size of a macro expansion to be always $\geq 3$.

To illustrate this further, consider the following example.
\begin{example}
    Consider the languages 
    \[\LL = \{x(\textcolor{blue}{a(}\textcolor{red}{c(d)}\textcolor{blue}{)},
    \textcolor{blue}{a(}\textcolor{red}{c(d)}\textcolor{blue}{)},
    \textcolor{purple}{b(}\textcolor{red}{c(d)}\textcolor{purple}{)},
    \textcolor{purple}{b(}\textcolor{red}{c(d)}\textcolor{purple}{)})\}\]
    and 
    \[\LL' = \{x(
    y(\textcolor{blue}{a(}\textcolor{red}{c(d)}\textcolor{blue}{)}),
    y(\textcolor{blue}{a(}\textcolor{red}{c(d)}\textcolor{blue}{)}),
    y(\textcolor{purple}{b(}\textcolor{red}{c(d)}\textcolor{purple}{)}),
    y(\textcolor{purple}{b(}\textcolor{red}{c(d)}\textcolor{purple}{)}))\}\]
    where $x$ has arity $4$, $d$ is constant and all other symbols are unary.
    Then the only minimal encoding for $\LL$ is $\fencode{\LL}{M}$ with
    \[\LL_M = \{x(\textcolor{blue}{a(}m\textcolor{blue}{)},
    \textcolor{blue}{a(}m\textcolor{blue}{)},
    \textcolor{purple}{b(}m\textcolor{purple}{)},
    \textcolor{purple}{b(}m\textcolor{purple}{)})\},\]
    \[\MM = \{m \mapsto \textcolor{red}{c(d)}\}\]
    whereas the only minimal encoding for $\LL'$ is $\fencode{\LL'}{M'}$ with
    \[\LL'_{M'} = \{x(m,m,m',m')\}, \]
    \[\MM' = \{m \mapsto y(\textcolor{blue}{a(}\textcolor{red}{c(d)}\textcolor{blue}{)}), m' \mapsto y(\textcolor{purple}{b(}\textcolor{red}{c(d)}\textcolor{purple}{)})\}.\]
    Note that while the subterm $c(d)$ occurs equally often in both languages it is only set up as a macro for $\LL$ but not for $\LL'$.
    Furthermore, these two examples showcase, that greedily instantiating minimal, or respectively maximal, macros in the Hasse diagram that reduce the size of the encoding does not necessarily result in a minimal encoding. Instead, it is likely that we need a more complicated argument over size and number of occurrences. This is left for future work.
\end{example}
\subsection{Implementation Details}

The algorithms described in Section~\ref{sec:algos} are primarily designed to support arguments we make in our proofs. However, an \textit{exact} implementation of these algorithms (as specified above in this Appendix)
will not lead to the desired performance in terms of processing time that we report. In this section, we provide specifications of slightly adapted algorithms that we use in our implementation, which exhibits the favourable performance behaviour in practice that we report in our paper.

The idea is based on a simple argument that avoids unnecessary iterations over the set of macro definitions. In Algorithm~
\ref{alg:maximalConstantLanguageCompression}, the outermost \texttt{while} and \texttt{for} loops in lines 4 and 5 correspond to a breadth-first traversal through the Hasse diagram of macros' fixed-point expansions w.r.t.\ the relation $\preceq$. This traversal, in combination with the \texttt{for} loop starting in line 7 over terms in $\Lang$, checks for \textit{every} macro $m$ whether it can be instantiated in \textit{any} term $t \in L$. Essentially, this traversal iterates over \textit{all pairs} $(m,t)$ of available macro symbols $m$ and terms $t$. 

However, this is obviously not necessary because the only macros that can possibly be instantiated in a given $t \in \Lang$ are macro symbols with a fixed-point expansion to \textit{subterms} of $t$. So, for a given term, it suffices to check macro symbols with a fixed-point expansion to a subterm of $t$. Such macros can be instantiated, as before, based on their containment relationship. This can drastically improve the running time in practice because there are often \textit{far} fewer subterms in a term than there are macro definitions.

The pseudo-code for this adapted strategy of instantiating macros is shown in Algorithm~\ref{alg:maximalConstantLanguageCompression-implementation}:
\begin{itemize}
    \item Line 4 constructs all subterms for a given $t \in \Lang$,
    \item Line 5 iterates over these subterms in descending order according to the the $\preceq$ relation, i.e., macro containment,
    \item Lines 6--10 performs the replacement of terms, i.e, instantiates a macro (note that the existence of $m \mapsto t'$ in $\M^*$ can be implemented by a simple look-up table mapping fixed-point expansions to macro symbols),
    \item Line 11 updates the iterator for terms in $T_s$ by removing the term tree $t'$ that has just been tested for macro instantiation in $t$.
\end{itemize}

The same strategy can be applied to adapt
Algorithm~\ref{alg:maximalConstantLanguageCompression-problem-2} in an analogous manner. Also note that we use this adapted algorithm in our approach to construct size-minimal encodings w.r.t.\ Problem 3 (without unary symbols).

We test equality between terms using the OWL API's\footnote{https://owlcs.github.io/owlapi/apidocs\_5/} \texttt{equals} operator for \texttt{OWLObjects} but implement the rewriting of terms based on their abstract syntax trees represented with data structures offered by JGraphT.\footnote{https://jgrapht.org/}

Another detail worth mentioning in Algorithm~\ref{alg:maximalConstantLanguageCompression-implementation} is that each $t \in \Lang$ can be macrofied \textit{independently} from any other $t' \in \Lang$. This means that $\Lang$ can be partitioned and a macrofication can be computed using \textit{parallel computing} techniques, making our approach not only faster but also scalable. Note, however, that we do not report on performance statistics using parallel computing techniques in this paper. This will be done in future work in which we explore possible optimization techniques in more detail.

\begin{algorithm}[H]
         \SetAlgoLined
           \LinesNumbered
        \setcounter{AlgoLine}{0}
         \SetKwFunction{algo}{algo}\SetKwFunction{proc}{macroIntroduction} 
         \KwResult{Minimal Encoding $(\Lang_M,\mathcal{M})$ of input language $\Lang$ w.r.t.\ input definitions $\mathcal{M}$}

         $\mathcal{M}^* \gets \{ m \mapsto \fixedpoint{m} \colon m \mapsto t \in \M \}$
         
         $\Lang_M \gets \Lang$\;

             \For{$t \in \Lang_M$}{

                 $T_s \gets \{ t' \colon t' \preceq t \}$

                  \For{$t' \in T_s$ s.t. there is no $t'' \in T_s$ with $t' \prec t''$}{
                     \If{there exist $m \mapsto t'$ in $\M^*$}{
                         \For{all $p$ with $t|_p = t'$} {
                         $t \gets t[m]_p$
                     }
                     
                     }

                     $T_s \gets T_s \setminus \{t'\}$
                    
                  }            
             }
         
         \Return{$(\Lang_M, \mathcal{M})$}
         \caption{Adaptation of Algorithm 1}\label{alg:maximalConstantLanguageCompression-implementation}
     \end{algorithm}

\subsection{Further Details on Experimental Results}

Four of the five used ontologies in this work are publicly available and can be acquired from BioPortal as cited in the paper. SNOMED is the only exception. However, researchers in many countries can acquire SNOMED with an academic license free of charge. So, even though SNOMED is not publicly available for download, SNOMED and can be acquired as an OWL file for research purposes.

Since our experimental evaluation is based on the experimental design used by~\cite{DBLP:conf/aaai/NikitinaK17}, all results are based on class expression axioms. In other words, we load an ontology and extract class expression axioms using the constructors \opr{SubClassOf}, \opr{EquivalentClasses}, \opr{DisjointClasses}, and \opr{DisjointUnion}. Furthermore, we disregard annotations, i.e., two axioms are considered identical if they are indistinguishable w.r.t.\ their logical structure.

In Table~\ref{tab:minimizing-expressions}, we provide a direct comparison between the results reported by~\cite{DBLP:conf/aaai/NikitinaK17} and our results w.r.t.\ Problem 3. While there are 7 out of 15 cases in which the approach proposed by~\cite{DBLP:conf/aaai/NikitinaK17} obtains higher percentages in terms of the average size difference between class expresions in the original ontology and the macrofied one (Red.by $E_n$), our approach outperforms the one by~\cite{DBLP:conf/aaai/NikitinaK17} in all but two cases in terms of the percentage of expressions that could be reduced (Red. $E_n$).

There seems to be a discrepancy w.r.t.\ the average size of expressions in the ontology Genomic CDS (cf.\ ontology C in Table~\ref{tab:minimizing-expressions}). It appears that~\cite{DBLP:conf/aaai/NikitinaK17} report a much higher average size, namely 14.2, compared to us, 5.4. This discrepancy might be an artefact of the slightly different experimental setup of both studies. While we measure \textit{all} occurring class expressions in the ontology, \cite{DBLP:conf/aaai/NikitinaK17} only sampled 100 expressions. In this context, it is important to know that Genomic CDS contains a few very large axioms, e.g., a \opr{DisjointClasses} axiom involving 385 named classes. If such a large axioms is part of the 100 sampled expressions, then this can result in a much larger average compared to the average of all class expressions in the ontology.
\begin{table*}[t]
  \centering
  \small
  \begin{tabular}{| c | c | r | r || r | r || r | r || r | r |}
      \hline

     P &
   O &
  \# Expr. &
  $\overline{S}_{\ont}$ &
   Red. $E_2$ &
  Red.by $E_2$ &
  Red. $E_5$ &
    Red.by $E_5$ &
     Red. $E_{10}$ &
    Red.by $E_{10}$\\
\hline
\hline
        NK & S & - & 5.4 & 36\% & 54\%  & 50\%  & 28\%  & 33\%   & 32\%\\
        3 & S  & \numprint{550810} & 5.32 & 85\% & 45\% & 96\%  & 45\% & 100\%  & 44\%\\
\hline
        NK & G & -  & 3.1 & 27\% & 43\% & 86\%  & 47\% & 100\%   & 72\%\\
        3 & G  & \numprint{39626}  & 2.93 & 53\% & 48\% & 71\%  & 51\% & 88\%   & 42\%\\
\hline
        NK & C & -   & 14.2 & 14\% & 76\% & 32\%  & 90\% & 37\%   & 90\%\\
        3 & C  & \numprint{3746}   & 5.40 & 26\% & 33\% & 97\%  & 34\% & 99\%   & 32\%\\
\hline
        NK & F & -  & 2.9 & 16\% & 48\% & 76\%  & 48\% & 0\%   & 0\%\\
        3 & F  & \numprint{28922}  & 3.92  & 83\% & 25\% & 100\% & 23\% & 100\% & 8\%\\
\hline
        NK & N & - & 3.0 & 5\% & 26\%  & 33\%  & 26\%  & 40\%   & 14\%\\
        3 & N  & \numprint{217866} & 2.36 & 58\% & 48\% & 91\%  & 53\% & 99\%   & 49\%\\
\hline

  \end{tabular}
  \caption{Comparison with \cite{DBLP:conf/aaai/NikitinaK17}: Statistics on minimising class expressions showing the number of class expressions (\# Expr.), the average size of class expressions in an ontology ($\Bar{S}_\ont$), the percentage of class expressions of size $n$ that could made smaller via macro instantiation (Red. $E_n$), and the average size difference between a class expression in $\ont$ and its macrofied counterpart (Red.by $E_n$). The letters NK in the first column indicate results reported by~\cite{DBLP:conf/aaai/NikitinaK17} that we repeat here for convenience.
  }
  \label{tab:minimizing-expressions}
\end{table*}

\subsection{Comments on Related Work}
While terms and term trees are usually defined to be indistinguishable in the context of tree languages, the arguments of terms and their corresponding branches in trees are generally assumed to be either all ordered or unordered, cf.~\cite{comon2008tree}.
However, there is little work on tree languages that include both ordered \textit{and} unordered trees. Moreover, I am not aware of any work that considers the combination of ranked, unranked, ordered, unordered, and mixed symbols (cf.~Definition~\ref{def:alphabetTermsLanguage}), as required by more complex languages such as OWL. 

An early version of the theoretical framework presented in this paper can be found in Chapter 6 of my PhD thesis~\cite{kindermann2021analysing}.
Even though I already consider terms
to be \textit{indistinguishable} from their corresponding term trees,
I define term trees to be \textit{ordered}.
To handle more complex languages, such as OWL, I simply
define \textit{term tree equality} via \textit{term equality} in the desired language.
The idea of defining term equality via isomorphisms between \textit{labelled unordered trees} (cf.~Definition~\ref{def:subterm}) was only fully worked out later \cite{DBLP:conf/semweb/KindermannS22}.

While the definitions presented in Section~\ref{sec:preliminaries} are sufficient to handle OWL, I need to point out that they don't capture \textit{all} details of OWL's structural specificartion. For example, the terms $t_1 = \opr{ObjectIntersectionOf}(\conc{A}, \conc{A})$
and $t_2 = \opr{ObjectIntersectionOf}(\conc{A})$ would be \textit{valid} terms in the context of our framework, with $t_1 \neq t_2$. However, the arguments of the OWL constructor $\opr{ObjectIntersectionOf}$ are specified as a \textit{set}. So, in OWL, we'd have $t_1 = t_2 = \opr{ObjectIntersectionOf}(\conc{A})$. Furthermore, OWL's specification requires $\opr{ObjectIntersectionOf}$ to have \textit{at least two} (distinct) arguments. This means that  $t_1$ and $t_2$ aren't even valid OWL expressions. Such issues can be avoided by defining a notion of equality between \textit{term trees} via equality of \textit{terms} in the desired language, as I have done in my thesis.

Another important difference between the work presented in this paper and my PhD thesis is \textit{macro nesting}. In my thesis, I consider macro systems that \textit{do not allow} for macro nesting, i.e, the expansion of every macro symbol is required to be its 1-step expansion. I thought to have proven that Problem~\ref{prob:macroSystems}, as defined in this work, can be solved in polynomial time w.r.t.\ such macro systems that don't allow for nesting.

However, the algorithm I proposed in my thesis for solving Problem~\ref{prob:macroSystems} w.r.t.~1-step-expansion macros is \textit{not sufficient} for the general case (as I later found out by discovering a counter-example).
Rather, the algorithm can only be seen as heuristic using a greedy strategy to construct smaller and smaller language encodings in each iteration. In the same vein, the presented correctness proof only shows the correctness of this greedy strategy 
(the argument in Lemma 6.3.4 "Compression Improvement" on page 181 overlooks that the introduction of larger macros, using the same argument as formulated here in  Lemma~\ref{lemma:macroPrecedence}, can make the instantiation of \textit{smaller} 1-step-expansion macros possible again. However, such smaller macros will not be considered to be re-introduced in the algorithm).

So, it remains an open question whether Problem~\ref{prob:macroSystems} can be solved in polynomial time w.r.t.\ macro systems only allowing for 1-step-expansion macros (no nesting).

One more detail worth mentioning is the topic of acyclicity. In my thesis, I do not require macros to be acyclic. In this work, however, we restrict our attention to acyclic macros so that the expansion of macros is guaranteed to exist, to be unique, and to not contain any macro symbols (cf.~Lemma~\ref{lemma:cycles}).

Also, our notion of acyclicity is \textit{not} based on entailment. Moreover, we do not subject the \textit{input language} of the minimisation problem to some notion of acyclicity,
since macro definitions are not assumed to be part of the input language.
This makes our approach more generally applicable compared to existing approaches specifically designed for ontologies.
That said, we can of course choose to \textit{interpret} parts of the input language as macro definitions, in which case this interpretation needs to satisfy said acyclicity constraints for macro definitions.

\newpage

\section*{Acknowledgements}
Anne-Marie George was  supported under the NRC Grant No~302203
``Algorithms and Models for Socially Beneficial Artificial Intelligence''.

\bibliography{aaai24}

\end{document}